\documentclass{article}

\PassOptionsToPackage{numbers, compress}{natbib}

\usepackage[preprint]{nips_2018}

\usepackage[utf8]{inputenc} 
\usepackage[T1]{fontenc}    
\usepackage{hyperref}       
\usepackage{url}            
\usepackage{booktabs}       
\usepackage{amsfonts}       
\usepackage{nicefrac}       
\usepackage{microtype}      

\usepackage{algorithm,algorithmic}
\usepackage{graphicx,url,xcolor}
\usepackage{amsmath,amssymb,amsopn,amsthm,bm}
\usepackage{multirow,rotating,subfigure}

\newcommand{\argmax}{\operatornamewithlimits{argmax}}

\providecommand{\norm}[1]{\bigl\lVert#1\bigl\rVert}
\newcommand{\wtop}{w^{\top}}
\newcommand{\bu}{\bar{u}}
\newcommand{\Y}{\mathcal{Y}}
\newcommand{\X}{\mathcal{X}}
\newcommand{\E}{\mathbb{E}}
\newcommand{\R}{\mathbb{R}}
\newcommand{\V}{\mathbb{V}}
\newcommand{\N}{\mathcal{N}}
\newcommand{\bigO}{\mathcal{O}}
\newtheorem{thm}{Theorem} 

\newtheorem{lemma}{Lemma}
\newtheorem{corr}{Corollary}

\title{Sparse Stochastic Zeroth-Order Optimization with an Application to Bandit Structured Prediction}

\author{
  Artem Sokolov\\
  Amazon Research \\
  \& Heidelberg University, Germany\\
  \texttt{sokolov@cl.uni-heidelberg.de}
  \And
  Julian Hitschler\\
  Computational Linguistics\\
  Heidelberg University, Germany\\
  \texttt{hitschler@cl.uni-heidelberg.de}
  \And
  Mayumi Ohta\\
  Computational Linguistics\\
  Heidelberg University, Germany\\
  \texttt{ohta@cl.uni-heidelberg.de}
  \And
  Stefan Riezler\\
  Computational Linguistics \& IWR\\
  Heidelberg University, Germany\\
  \texttt{riezler@cl.uni-heidelberg.de}
}

\begin{document}

\maketitle

\begin{abstract}
  Stochastic zeroth-order (SZO), or gradient-free, optimization allows to optimize arbitrary functions by relying only on function evaluations under parameter perturbations, however, the iteration complexity of SZO methods suffers a factor proportional to the dimensionality of the perturbed function. We show that in scenarios with natural sparsity patterns as in structured prediction applications, this factor can be reduced to the expected number of active features over input-output pairs. We give a general proof that applies sparse SZO optimization to Lipschitz-continuous, nonconvex, stochastic objectives, and present an experimental evaluation on linear bandit structured prediction tasks with sparse word-based feature representations that confirm our theoretical results.
\end{abstract}

\section{Introduction}
\label{sec:intro}
Random gradient-free methods \cite{NesterovSpokoiny:15} provide a simple approach to optimization by applying Gaussian smoothing to an arbitrary function, thus establishing Lipschitz-continuity of the gradient of the smoothed function, and allowing to approximate the gradient by comparisons of function values obtained at randomly perturbed parameter vectors. In optimization theory, related techniques are also named ``zeroth-order'' \cite{GhadimiLan:12,DuchiETAL:15} or ``derivative-free'' \cite{JamiesonETAL:12,Shamir:13} methods since rather than first- or second-order derivatives, only function values are accessible. The main advantage of stochastic zeroth-order (SZO) methods lie in their applicability to optimization of non-differentiable functions, or in black box situations where nothing but function values are available. The main disadvantage compared to stochastic first-order (SFO), or gradient-based, techniques is a dependency of the convergence speed on the dimensionality of the function to be evaluated.

In machine learning, obtaining a noisy realization of reward or loss function values at a proposed parameter perturbation corresponds to learning with ``bandit feedback'' \cite{FlaxmanETAL:05,AgarwalETAL:10}. These methods have found renewed interest in the area of reinforcement learning, notably in applications where gradient information is available and SFO techniques would be applicable. The cited advantages of SZO methods for reinforcement learning are their simplicity and robustness against hyperparameter changes \cite{ManiaETAL:18}, the fact that they are highly parallelizable and do not require backpropagation \cite{SalimansETAL:17}, or the improved exploration behavior in parameter space (instead of in action space) \cite{SehnkeETAL:10}.

The goal of this paper is to show that the bottleneck of SZO techniques---the dependency of convergence speed on the dimensionality of the parameter space being perturbed---can be reduced to the expected number of active features over input-output pairs in structured prediction with sparse feature spaces. Such sparsity patterns can be found in sequence labeling tasks that represent sequences by context word representations. An example is the task of noun-phrase chunking where a natural sparsity pattern is defined by n-grams of input words and output chunk labels that can possibly occur in given input-output pairs. Our paper starts with a general proof that adapts the work of \cite{NesterovSpokoiny:15} to the case of sparse parameter perturbation for Lipschitz-continuous (but not necessarily Lipschitz-smooth), nonconvex, stochastic functions. We present three algorithms that instantiate our theory to structured prediction applications, one based on \cite{NesterovSpokoiny:15}'s standard two-point function evaluation method, and two methods that aim to reduce feedback complexity to a comparison of function values, or to a comparison of one-point feedback to a running average, with applications to human feedback in mind. In our experimental evaluation, we compare our SZO techniques to standard SFO techniques for bandit structured prediction tasks from natural language processing \cite{SokolovETALnips:16}. Our experimental results on the task of noun-phrase chunking show that the convergence speed of all compared SZO techniques is improved by applying sparse parameter perturbations, reaching a performance close to the standard SFO technique. Furthermore, we present an experimental evaluation on the task of statistical machine translation that investigates the use of a non-differentiable maximum-a-posteriori (MAP) criterion at training and test time in SZO techniques---something that is not possible in SFO methods. On this task our best SZO result outperforms the standard SFO technique, pointing to another possible advantage of SZO techniques.

\section{Related Work}
\label{sec:related}
Zeroth-order, or gradient-free, stochastic optimization dates back to the finite-difference method for gradient estimation of \cite{KieferWolfowitz:52} where the value of each component of $w$ is perturbed separately while holding the other components at nominal value. This technique has since been replaced by more efficient methods based on simultaneous perturbation of all weight vector components \cite{Spall:92,Spall:03,KushnerYin:03}. More recent developments apply the simultaneous perturbation principle to bandit learning \cite{FlaxmanETAL:05,YueJoachims:09,AgarwalETAL:10,DuchiETAL:15} and extend it from (strongly) convex function to non-convex functions \cite{GhadimiLan:12,NesterovSpokoiny:15}. To our knowledge, the application to structured prediction problems, especially regarding sparse perturbations in active feature space, is novel.

Connections of SZO methods to evolutionary algorithms and reinforcement learning have first been described in \cite{Spall:03}. Recent work has applied SZO techniques successfully to policy gradient methods for deep reinforcement learning \cite{SehnkeETAL:10,SalimansETAL:17,PlappertETAL:18,ManiaETAL:18}.
The framework of bandit structured prediction \cite{SokolovETALnips:16,KreutzerETAL:17,NguyenETAL:17,BahdanauETAL:17} is closely related to policy gradient techniques in reinforcement learning \cite{Williams:92,SuttonETAL:00,KondaTsitsiklis:00}, for example, linear bandit structured prediction is termed ``gradient bandits'' in \citet{SuttonBarto:17}. Earlier work in structured prediction has applied SZO techniques for expected loss minimization (a.k.a. minimum risk training) \cite{McAllesterKeshet:11,KeshetETAL:11}.

The quality of SZO methods to provide improved exploration behavior at reduced variance is also appreciated in recent reparameterization approaches. Here a loss function is reparameterized via Gaussian smoothing, and loss values obtained under latent variable perturbations are combined with backpropagation of first-order derivatives \cite{KingmaWelling:14,JangETAL:17,PlappertETAL:18}. Our work could provide an alternative to the combination of SZO and SFO techniques to an end-to-end application of SZO techniques in reparameterization scenarios.

\section{Sparse Zeroth-Order Stochastic Optimization for Nonconvex Objectives}
\label{sec:theory}
In the following, we give a theoretical analysis of gradient-free optimization by sparse parameter perturbation for Lipschitz-continuous (but not necessarily Lipschitz-smooth), nonconvex, and stochastic objectives. Our analysis builds on \cite{NesterovSpokoiny:15}.

\subsection{Problem Statement}

We would like to solve a stochastic optimization problem
\begin{align}
 \min_w f(w), \text{ where } f(w) := \E_x[F(w,x)],
  \label{eq:szo-objective}
\end{align}
and $\E_x$ denotes the expectation over inputs $x \in \X$, and $w \in \R^n$ is the dimensionality of the weight vector parameterizing the objective function.
We address the the general case of non-convex functions $F$ for which we furthermore assume Lipschitz-continuity\footnote{We use \cite{Nesterov:04}'s notation of function classes where $C^{k,p}$ denotes the class of $k$ times differentiable functions whose $p$th derivative is Lipschitz continuous.}, i.e.,  $F(w,x) \in C^{0,0}$
iff
\begin{align}
  |F(w,x) - F(w',x)| \leq L_0(F(\cdot,x)) \norm{w - w'}, \forall w, w', x.
  \label{eq:lipschitz}
\end{align}

\cite{NesterovSpokoiny:15} show how to achieve a smooth version of an arbitrary function $f(w)$ by Gaussian blurring that assures continuous derivatives everywhere in its domain. In their work, random perturbation of parameters is based on $n$-dimensional Gaussian random vectors $u$ from a zero-mean isotropic multivariate Gaussian with unit $n \times n$ covariance matrix $\Sigma = I$ s.t.
\begin{align}
\N(u)=\frac{1}{\sqrt{(2\pi)^n\cdot \det\Sigma}} e^{-\frac{1}{2}u^\top\Sigma^{-1}u},
\label{eq:gaussian}
\end{align}
and a Gaussian approximation of $f(w)$ is defined as $f_\mu(w) = \E_{u}[f(w + \mu u)]$, where $\mu \geq 0$ is a smoothing parameter. Furthermore, a Lipschitz-continuous gradient even for a non-differentiable original function $f$ can be given by applying standard differentiation rules to $f_\mu(w)$, yielding $\nabla_w f_\mu(w) = \E_{u}[\frac{f(w + \mu u) - f(w)}{\mu}u]$ (see \cite{NesterovSpokoiny:15}, eq. (21)).

In the case of linear stochastic structured prediction, the functional $f(w)$ is an expectation over inputs $x$, and the function $F(w,x)$ is defined with respect to a linear model $\wtop \phi(x,y)$ where $\phi: \X\times\Y\rightarrow\R^n$ is a joint, possibly sparse, feature representation on inputs and outputs. We express sparsity in parameter perturbation by restricting the Gaussian random vector $u \in \R^n$ to the active, i.e. non-zero, features for each input $x$, where $\restriction$ is the restriction operator s.t.
\begin{align}
  \N(u,x) = \N(u)\restriction_{u_i \text{ s.t. } \exists y \in Y(x), \phi_i(x,y) \neq 0}.
  \label{eq:sparse-gaussian}
\end{align}
We denote by $\bu \in \R^n$ the Gaussian random vector resulting from a sparse perturbation, and by
\begin{align}
  \bar{n}(x) := \norm{\bu}_0 = |\bu_1|^0 + \ldots + |\bu_n|^0, \text{ where } 0^0 = 0,
  \label{eq:zero-norm}
\end{align}
the effective number of parameter perturbations for an input $x$. Based on the notion of sparse perturbation vectors $\bu \in \R^n$, we can redefine the Gaussian approximation as
\begin{align}
  f_\mu(w) = \E_{\bu}[f(w + \mu \bu)].
\label{eq:f_mu_bu}
\end{align}

\subsection{Convergence Analysis}

Our first Lemma applies standard differentiation rules to the the continuous function $f_\mu(w)$ in eq. \eqref{eq:f_mu_bu}, yielding a Lipschitz-continuous gradient. A full proof, adapting the calculations given in \cite{NesterovSpokoiny:15}, eq. (21), to our case is given in the supplementary material.

\begin{lemma}
  \label{lemma:gradients}
  \begin{align}
    \nabla_w f_\mu(w) & = \E_{\bu,x}[\frac{F(w + \mu \bu,x) - F(w,x)}{\mu}\bu]. \label{eq:2-point}
    \end{align}
  \end{lemma}
\cite{NesterovSpokoiny:15} show that for Lipschitz-continuous functions $F$, the variance of the gradient approximation can be bounded by the Lipschitz constant and by the norm of the random perturbation.The term $\E_{u}[\norm{u}^p]$ can itself be bounded by a function of the exponent $p$ and the dimensionality $n$ of the function space. This is how the dependency on $n$ enters iteration complexity bounds and where an adaptation to sparse perturbations has to chime in.
In order to adapt \cite{NesterovSpokoiny:15}'s bounds to sparse gradient-free optimization, we first need to match their bounds on the norms of the random perturbation. The simple case of the squared norm of the random perturbation given below illustrates the idea. If a coordinate $i$ in feature space is not perturbed, no variance $\V_{\bu}[\bu_i]$ in incurred. The smaller the variance, the smaller the factor $\bar{n}(x)$ that directly influences iteration complexity bounds:
\begin{align*}
\E_{\bu}[\norm{\bu}^2]  & = \E_{\bu}[\bu_1^2 + \bu_2^2 + \ldots + \bu_n^2] 
     = \E_{\bu}[\bu_1^2] + \E_{\bu}[\bu_2^2] + \ldots + \E_{\bu}[\bu_n^2] \\
     & = \V_{\bu}[\bu_1] + \V_{\bu}[\bu_2] + \dots + \V_{\bu}[\bu_n] 
     = \bar{n}(x).
\end{align*}

For the general case of $\E_{\bu}[\norm{\bu}^p]$, $p \geq 2$, we match \cite{NesterovSpokoiny:15}'s equation (17). We shorten the proof to the parts relevant to sparse perturbations.

\begin{lemma}
  \label{lemma:u_p}
  \begin{align}
    \E_{\bu}[\norm{\bu}^p] \leq (p + \bar{n}(x))^{p/2}.
    \end{align}
  \end{lemma}
\begin{proof}
  \begin{align*}
    \E_{\bu}[\norm{\bu}^p]
    & = (2\pi)^{-n/2}\int_{\bu} \norm{\bu}^p e^{-\frac{1}{2}{\bu}^\top {\bu}}\\
    & = (2\pi)^{-n/2}\int_{\bu} \norm{\bu}^p e^{-\frac{\tau}{2}{\bu}^\top {\bu}} e^{-\frac{1-\tau}{2}{\bu}^\top {\bu}}\\
    & \leq (2\pi)^{-n/2}\int_{\bu} \Big(\frac{p}{\tau e}\Big)^{p/2} e^{-\frac{1-\tau}{2}{\bu}^\top {\bu}}\\
    & = \Big(\frac{p}{\tau e}\Big)^{p/2} (2\pi)^{-n/2}\int_{\bu} e^{-\frac{1}{2}{\bu}^\top \big(\frac{I}{1-\tau}\big)^{-1}{\bu}}\\
    & = \Big(\frac{p}{\tau e}\Big)^{p/2} (2\pi)^{-n/2}\sqrt{(2\pi)^n\cdot \det\frac{I}{1-\tau}}\\
    & = \Big(\frac{p}{\tau e}\Big)^{p/2} \frac{1}{(1-\tau)^{\bar{n}(x)/2}}\\
    & \leq (p + \bar{n}(x))^{p/2}.
  \end{align*}
  The first inequality follows from $t^p e^{-\frac{\tau}{2}t^2} \leq \Big(\frac{p}{\tau e}\Big)^{p/2}$, for $t \geq 0$ (see \cite{NesterovSpokoiny:15}, eq.~(80)). The second inequality follows by minimizing the right-hand side in $\tau \in (0,1)$ (see \cite{NesterovSpokoiny:15}, Lemma 1).
\end{proof}
Lemma \ref{lemma:u_p} applies the idea illustrated above to higher order norms of random perturbations: If a coordinate is not perturbed, the determinant of the covariance matrix reduces to a product of variances of the active features. This allows us to bound the perturbation factor for each input by $\bar{n}(x) \ll n$. Our convergence theorem analyzes the following SZO algorithm with Sparse Perturbations (SZO-SP).

\begin{algorithm}[H]
  \caption{SZO-SP}
  \label{alg:SZO-SP}
  \begin{algorithmic}[0]
    \STATE Input: sequence of learning rates $h_k$, smoothing parameter $\mu > 0$.
    \STATE Initialize: $w_0 = 0$
    \FOR{$k \geq 0$}
    \STATE For each $w_k$, sample $x_k$ and $\bu_k$.
    \STATE Compute $s_\mu(w_k) := \frac{F(w_k + \mu \bu_k,x_k) - F(w_k,x_k)}{\mu}\bu_k$.
    \STATE Update $w_{k+1} = w_k - h_k s_\mu(w_k)$.
    \ENDFOR
  \end{algorithmic}
\end{algorithm}

We present an analysis for nonconvex functionals $F(w,x)$ and $f(w) = \E_x[F(w,x)]$. Furthermore, we assume that each $F(\cdot,x) \in C^{0,0}$ with $L_0(F(\cdot,x)) \leq L_0$, and that $f \in C^{0,0}$. 
\begin{thm}
  \label{thm:convergence}
  Assume a sequence $\{w_k\}_{k \geq 0}$ be generated by Algorithm \ref{alg:SZO-SP}. Let $f(w) \geq f^\ast$, $\forall w \in \R^n$, and define $\bar{n} \geq \E_{x_k}[\bar n(x_k)], \; \forall k \geq 0$ and $S_N := \sum_{k=0}^N h_k$. Furthermore, let $L_1$ denote the Lipschitz constant of $\nabla f_\mu(w)$, and let $\bar{\mathcal{U}}_k = (\bu_0, \ldots, \bu_k)$ and $\mathcal{X}_k =  (x_0, \ldots, x_k)$. Then for any $N > 0$, we have
  \begin{align}
  \label{eq:bound}
    \frac{1}{S_N} \sum_{k=0}^N h_k \E_{\bar{\mathcal{U}}_k,\mathcal{X}_k }\big[\norm{\nabla{f_\mu(w_k)}^2}\big] \leq
      \frac{1}{S_N} \left(
      (f_\mu(w_0) - f^\ast)
      + \left(\frac{1}{2} L_1 (\bar{n}+4)^2 L_0^2 \right) \sum_{k=0}^N h_k^2
      \right). \\\notag
    \end{align}
  \end{thm}
\begin{proof}
  The proof uses the fact that for a Gaussian approximation $f_\mu(w)$, its gradient is Lipschitz-continuous even if the gradient of $f(w)$ is not. For Lipschitz constant $L_1$ of $\nabla f_\mu(w)$, we have
  \begin{align*}
    f_\mu(w_{k+1}) & - f_\mu(w_k) - \left< \nabla f_\mu(w_k), w_{k+1} - w_k \right> \leq \frac{1}{2} L_1 \norm{w_k - w_{k+1}}^2.
  \end{align*}
  Applying an update step of Algorithm \ref{alg:SZO-SP} lets us rewrite $w_{k+1}$ as $w_k - h_k s_\mu(w_k)$, leading to
  \begin{align*}
    f_\mu(w_{k+1}) & \leq  f_\mu(w_k) - h_k \left< \nabla f_\mu(w_k), s_\mu(w_k) \right>  + \frac{1}{2} L_1 h_k^2 \norm{s_\mu(w_k)}^2.
  \end{align*}
  Taking expectations in $\bu_k$ and $x_k$, we can apply equation~\eqref{eq:2-point}
  , and get
  \begin{align*}
    \E_{\bu_k,x_k}[f_\mu(w_{k+1})] & \leq  f_\mu(w_k) - h_k \norm{\nabla f_\mu(x_k)}^2  + \frac{1}{2} L_1 h_k^2 \norm{s_\mu(w_k)}^2.
  \end{align*}
  
  Furthermore, the expected squared norm of $s_\mu(w_k)$ can be bounded by
  \begin{align*}
    \E_{\bu_k,x_k} \big[\norm{s_\mu(w_k)}^2\big] & = \frac{1}{\mu^2} \E_{\bu_k,x_k} \big[\left( F(w_k + \mu \bu_k,x_k) - F(w_k,x_k) \right)^2 \norm{\bu_k}^2\big] \\
    & \leq \frac{1}{\mu^2} \E_{\bu_k,x_k} \big[\norm{w_k + \mu \bu_k - w_k}^2 L_0^2 \norm{\bu_k}^2\big] \\
    & = \E_{\bu_k,x_k} \big[L_0^2 \norm{\bu_k}^4\big] \leq L_0^2 (\bar{n}+4)^2.
  \end{align*}
The first inequality follows by the assumption of Lipschitz continuity of all $F(\cdot,x) \in C^{0,0}$ with $L_0(F(\cdot,x)) \leq L_0$. The second inequality follows by applying Lemma \ref{lemma:u_p} for $p=4$ and taking the expectation $\E_{x_k}[\bar{n}(x_k)]$ whose upper bound is denoted by $\bar{n}$. 
This yields the following inequality
  \begin{align*}
    \E_{\bu_k,x_k}[f_\mu(w_{k+1})] & \leq  f_\mu(w_k) - h_k \norm{\nabla f_\mu(x_k)}^2 + \frac{1}{2} L_1 L_0^2 (\bar{n}+4)^2 h_k^2.
  \end{align*}
Taking expectations over $\bar{\mathcal{U}}_k$ and $\mathcal{X}_k $, and summing up over $k=0, \ldots, N$ yields the result.
\end{proof}
Theorem \ref{thm:convergence} gives a non-asymptotic bound on the expected squared gradient norm for any sequence of iterates of Algorithm \ref{alg:SZO-SP}. 
\cite{NesterovSpokoiny:15}, Section 7, furthermore show that by an appropriate choice of learning rates $h_k$ and smoothing parameters $\mu$, the iteration complexity for nonconvex zeroth-order optimization, i.e., the number of iterations necessary to guarantee an accuracy of $\epsilon$ for the expected squared norm of the gradient of $f_\mu$, can be analyzed as $\bigO{(\frac{n^3}{\epsilon^2})}$.
The same algebraic manipulations can be applied to result \eqref{eq:bound} that is adapted to sparse perturbations, leading to the following Corollary:

\begin{corr}
  \label{corr:complexity}
  \begin{align}
    \E_{\bar{\mathcal{U}}_N,\mathcal{X}_N }[\norm{\nabla f_\mu(w_{N})}^2] \leq \epsilon \text{ if } N \geq \bigO{\big(\frac{\bar{n}^3}{\epsilon^2}\big)}.
      \end{align}
\end{corr}
Corollary \ref{corr:complexity} shows that the factor $n^3$ that the $\epsilon$-accuracy of SZO methods for nonconvex optimization suffers in comparison to nonconvex SFO optimization can be reduced to the factor $\bar{n}^3$ that can benefit from strong sparsity patterns. A full proof, adapting the calculations given in \cite{NesterovSpokoiny:15}, Section 7, to our case is given in the supplementary material.

\section{Algorithms for Bandit Structured Prediction}
\label{sec:algo}
\paragraph{Update Rules.} Algorithm~\ref{alg:SZO-SP} defines an update rule by a \emph{two-point function evaluation}. Two-point update rules have been introduced as simultaneous perturbation gradient approximation by \cite{Spall:92}, and have later become standard in gradient-free optimization \cite{GhadimiLan:12,DuchiETAL:15}.

A possibility to reduce feedback complexity, with applications that obtain feedback from humans in mind, is to ask for a boolean-valued, relative comparison of function values. Algorithm~\ref{alg:SZO-SP} can be modified easily to use the following \emph{function comparison} update rule:
\[
\text{ If } \; F(w_k + \mu \bu_k,x_k) < F(w_k,x_k),
\text{ Update } \; w_{k+1} = w_k + \frac{h_k}{\mu} \; \bu_k.
\]
This rule can be seen as an SZO alternative to the dueling bandits algorithm of \cite{YueJoachims:09}.

A similar effect as a two-point function evaluation can be achieved by comparing a one-point function evaluation against a running average of function evaluations performed so far. This technique is known as control variates in Monte Carlo simulation \cite{Ross:13}. The idea is to augment a random variable $X$ whose expectation is sought, by another random variable $Y$ to which $X$ is highly correlated. Let $Y$ denote the control variate, and let $\bar{Y}$ denote its expectation. Then the quantity $X-\,Y + \bar{Y}$ is an unbiased estimator of $\E[X]$. The variance reduction effect of control variates can be seen by computing the variance of this quantity: $\text{Var}(X-Y) = \text{Var}(X) + \text{Var}(Y) -2\text{Cov}(X,Y).$
Choosing a control variate such that $\text{Cov}(X,Y)$ is positive and large enough, the variance of the gradient estimate will be reduced. In our case, the random variable a one-point gradient approximation evaluated at a sampled $x_k$ and $\bu_k$. A well-known control variate from reinforcement learning \cite{Williams:92} incorporates the average cumulative reward (or loss) as so-called baseline into the estimator, yielding a \emph{baseline comparison} update rule:
\[
\text{ Compute } \; Y_k = \frac{1}{k} \sum_{j=1}^k F(w_j + \mu \bu_j,x_j) \; \bu_k,
\text{ Update } \; w_{k+1} = w_k - \frac{h_k}{\mu} \left( F(w_k + \mu \bu_k,x_k) - Y_k \right) \; \bu_k.
\]
Note the similarity of the above rule to the two-point feedback rule where $Y_k$ plays the role of a slowly changing unperturbed function value. A similar rule has been used in gradient-free reinforcement learning \cite{SehnkeETAL:10}.

\paragraph{Linear Structured Prediction Models.} One possibility to instantiate the algorithms described above to structured prediction is to encode a task loss evaluation under MAP prediction as initial function $F$. In this work, we assume MAP prediction under a linear model $\hat{y}(x,w) = \argmax_{y \in \Y(x)} \wtop \phi(x,y)$. Let $\Delta:\Y \rightarrow [0,1]$ be a task loss function for structured prediction, e.g. 1-BLEU score for machine translation, then the initial function $F$ can be defined as 
\begin{align}
  F(w,x) := \Delta(\hat{y}(x, w)).
  \label{eq:map}
\end{align}
This criterion has been used in \cite{McAllesterKeshet:11,KeshetETAL:11}. A similar deterministic criterion has been used for SZO optimization for deep reinforcement learning by \cite{SalimansETAL:17,ManiaETAL:18}.

Note that the MAP criterion \eqref{eq:map} is not only not differentiable, but also not Lipschitz-continuous. This discontinuous criterion can be replaced by the smooth criterion \eqref{eq:annealed} that is be obtained by computing an annealed expected loss criterion under temperature parameter $\gamma \geq 0$:
\begin{align}
  F(w,x) & := \E_{p_{w,\gamma}(y|x)}[\Delta(y)] = \sum_{y \in \Y(x)} \Delta(y) \frac{(\exp \wtop \phi(x,y))^\gamma}{\sum_{y \in \Y(x)} (\exp \wtop \phi(x,y))^\gamma}. 
  \label{eq:annealed}
\end{align}
This criterion approaches criterion \eqref{eq:map} for $\gamma \rightarrow \infty$ (see \cite{SmithEisner:06}), thus both criteria can be used interchangeably in experiments.

\paragraph{Stochastic First-Order Optimization.} In our experiments, we will compare SZO approaches to gradient-based SFO algorithms. The latter will serve as upper bound in terms of convergence speed. SFO algorithms for linear bandit structured prediction have been introduced in \cite{SokolovETALnips:16}. We employ an algorithm that optimizes an expected loss objective $\E_{p(x)p_{w}(y|x)}[\Delta(y)]$ by performing gradient descent on the stochastic gradient $s(w) =  \Delta(y) \frac{\partial \log p_{w}(y|x)}{\partial w}$.

\section{Experiments}
\label{sec:exps}
\paragraph{Experimental Design.}

The structured prediction tasks in our experiments have been established for bandit structured prediction using SFO optimization in \cite{SokolovETALnips:16}. The task of noun-phrase chunking uses high-dimensional, but very sparse, word-based feature representations. The task of statistical machine translation (SMT) is a sequence-to-sequence prediction problem using a small, dense feature representation. The goal of the latter task is to investigate potential gains of using the same, non-differentiable, criteria at training and test time. All tasks are based on linear models.

Training for all tasks was done by supervised-to-bandit conversion where bandit feedback is simulated by evaluating $\Delta$ against gold standard structures which are never revealed to the learner. $\Delta$ is a loss function obtained from a task reward, namely 1-BLEU score at sentence level \cite{NakovETAL:12} for SMT, and 1-F1 score for chunking.

Convergence speed is evaluated by plotting the average cumulative loss against iterations. In our experiments, we use the MAP criterion \eqref{eq:map} and define average cumulative loss/reward at iteration $t$ as
  $
  \overline{\Delta}_t =\frac{1}{t} \sum_{k=1}^t \Delta(\hat{y}_k(x_k,w_k + \mu \bu_k)).
  $
This criterion corresponds to regret under the assumption of zero loss for the optimal model.

Test set evaluation is done by the standard offline evaluation for the respective tasks. Machine translation is evaluated by measuring the corpus-based BLEU score \cite{Papineni:02} against an unseen test set of human reference translations. Chunking is evaluated by F1 score on an unseen test set. The early stopping point for test set evaluation is chosen according to a standard online-to-batch conversion by selecting the model that performs best on development data for final evaluation on test data.
All evaluation results are averaged over three runs with different random seeds.
Further comparison points are out-of-domain lower bounds, in-domain upper bounds, and the SFO algorithm (Expected Loss Minimization) of \cite{SokolovETALnips:16}.

\paragraph{Sparse Models for Noun-Phrase Chunking.} We followed \cite{ShaPereira:03} in applying a linear conditional random field (CRF) model to the noun phrase chunking task on the CoNLL-2000 dataset. The original training set was split into a development set (top 1,000 sent.) and a training set (7,936 sent.); the test set was kept intact (2,012 sent.). Training for bandit learning on the chunking task is done by cold starting the models from $w_0=\mathbf{0}$.

For an input sentence $x$, each CRF node $x^i$ carries an observable word and its part-of-speech tag, and has to be assigned a chunk tag $c^i$ out of 3 labels: \textbf{B}eginning, \textbf{I}nside, or \textbf{O}utside (of a noun phrase). Chunk labels are not nested.  As in \cite{ShaPereira:03}, we use second order Markov dependencies (bi-gram chunk tags), such that for sentence position $i$, the state is $y^i=c^{i-1}c^i$, increasing the label set size from 3 to~9. The model uses feature templates that combine these labels with uni-, bi-, and tri-grams of Part-of-Speech tags, and with uni- and bi-grams of words, leading to high sparsity pattern of on average 0.25\% active features for over 1.5M features on the training set.

\paragraph{Dense Feature Models for SMT.} The learning goal in our SMT experiment is re-ranking of $n$-best translation lists of size 5,000 using a linear combination of 14 dense features. The experiments are based on the  \texttt{cdec} \cite{DyerETAL:10} framework. The experimental setup is French-to-English domain adaptation from Europarl to NewsCommentary domains using the data of \cite{KoehnSchroeder:07}. 

The bandit learning algorithms were initialized with the learned weights of the out-of-domain median model and used 40,444 parallel in-domain sentence pairs. Bandit feedback was simulated by evaluating the sampled translation against the reference using as loss function $\Delta$ a smoothed per-sentence $1 - \textrm{BLEU}$ (zero $n$-gram counts being replaced with $0.01$). The possible range of improvements is given by the difference in performance of $0.257$ BLEU for the out-of-domain model, compared to $0.284$ BLEU for an in-domain model, evaluated on a separate in-domain test set of 2,007 parallel sentence pairs.

\begin{figure}[t]
    \centering
    \includegraphics[width=0.45\columnwidth]{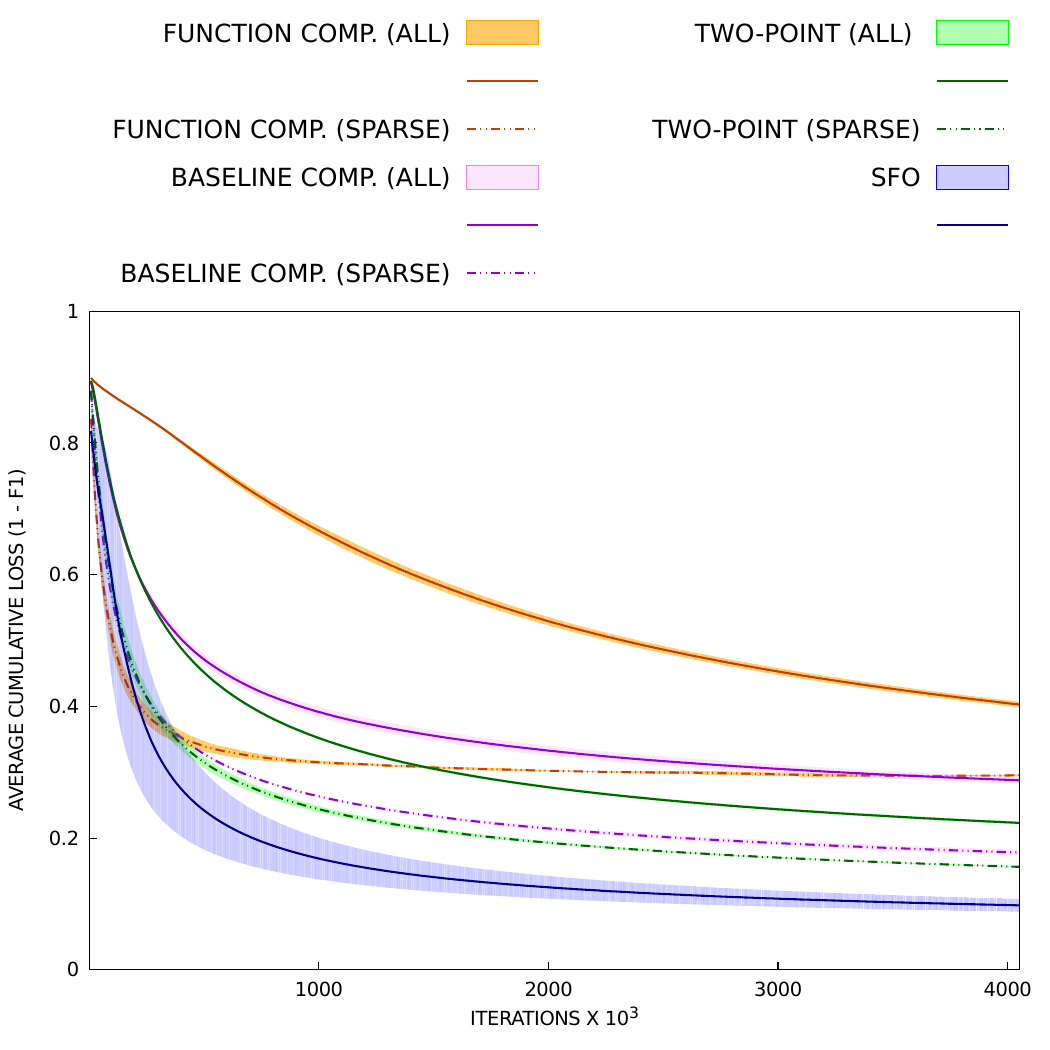}
    \includegraphics[width=0.45\columnwidth]{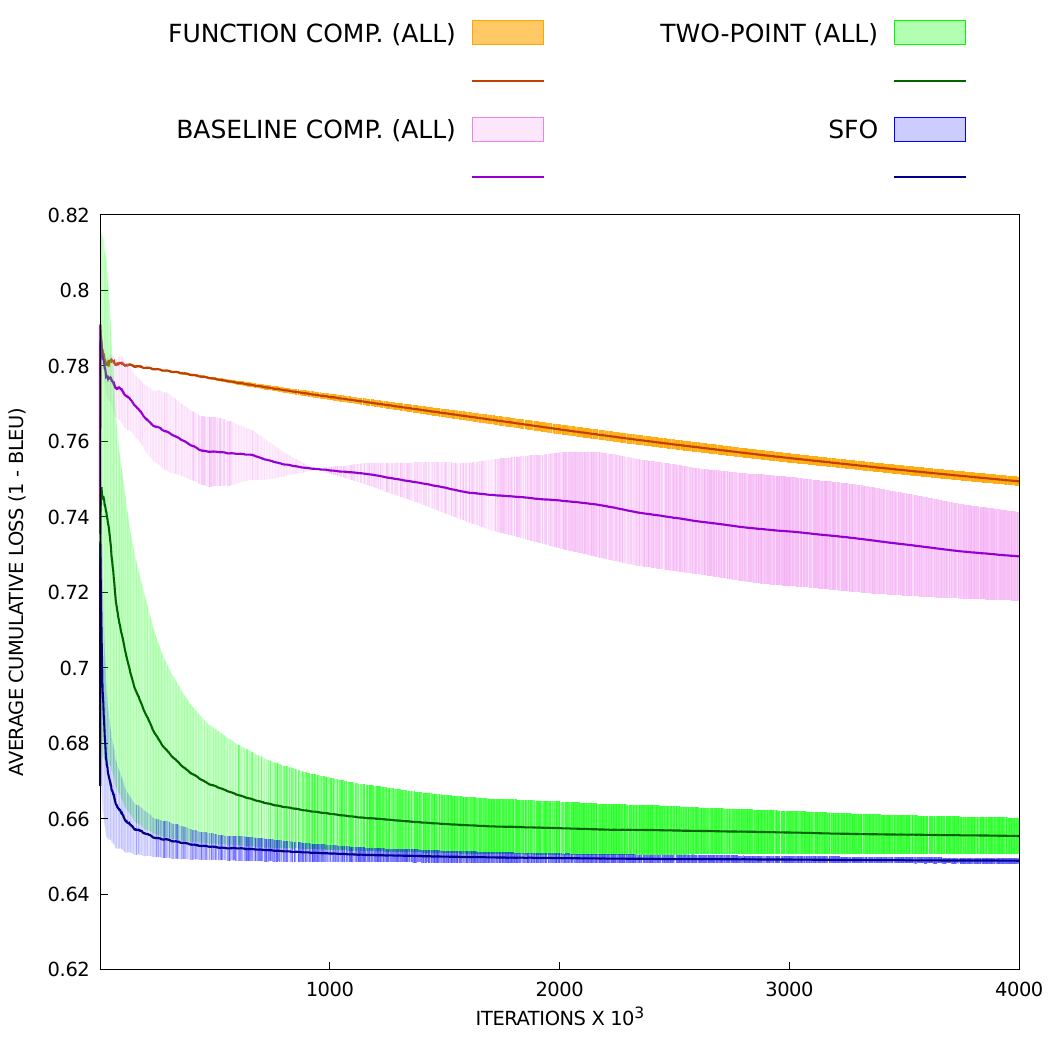}
    \caption{Average cumulative loss on training data for sparse noun phrase chunking (left) and dense statistical machine translation (right). All results were obtained with fixed hyperparameters for each task, averaged over 3 runs with different random seeds for each algorithm, showing mean results in bold lines, 2 standard deviations in filled areas.}
    \label{fig:avg_cum_loss}
\end{figure}

\paragraph{Experimental Results for Chunking.} The left plot in Figure \ref{fig:avg_cum_loss} confirms our theoretical findings by showing faster convergence for {SPARSE} perturbations (dashed curves) over perturbation of {ALL} parameters (solid curves) for each update rule. SPARSE perturbations for \emph{two-point} updates lead to a similar convergence speed as the SFO algorithm, followed by SPARSE perturbations for \emph{baseline comparison} updates, and SPARSE perturbation for \emph{function comparison}. The hyperparameters for all algorithms were kept fixed at constant learning rate $h=0.01$ and exploration parameter $\mu=0.01$.

Test set results for chunking are shown Table \ref{tab:test_results}. The results were obtained by tuning hyperparameters on development sets for constant learning rate $h$ in the range of $10^{-2}$ to $10^{-3}$, and exploration parameter $\mu$ in the range of $10^{-1}$ to $10^{-2}$. Best development settings for all algorithms were obtained close to the time horizon of 4M iterations, showing an undertraining behavior in all cases. However, we find the same relations in test set performance as were obtained for convergence speed, with SPARSE \emph{two-point} updates coming close to the SFO result.

\paragraph{Experimental Results for Machine Translation.} Figure \ref{fig:avg_cum_loss} depicts the convergence behavior of SZO optimization for the dense SMT re-ranking task, using perturbations of ALL features. Out of the SZO update rules, \emph{two-point} update lead to fastest convergence, followed by \emph{baseline comparison} and \emph{function comparison} update rules. All average cumulative loss results were obtained under the same hyperparameter settings of $h=0.001$ and $\mu=0.001$ for all algorithms.

Test set results are shown in \ref{tab:test_results}. Here optimal hyperparameter settings were adjusted on a development set of 1,064 parallel in-domain sentence pairs for constant learning rate $h$ in the range of $10^0$ to $10^{-5}$, and exploration parameter $\mu$ in the range of $10^{-2}$ to $10^{-6}$. Final results are obtained by averaging three independent runs using the hyperparameters found on the development set. Best development settings for all algorithms are obtained at 1-4M iterations, except for two-point feedback that reached an optimum already around 100k iterations.
The goal of the SMT experiment was to show a possible advantage of using the same non-differentiable MAP criterion at training and test time for SZO algorithms. As we can see, all SZO results outperform the SFO result.
In the range of 2.7 BLEU points between out-of-domain and in-domain models, we can achieve an improvement of 1.6 BLEU points by SZO with \emph{two-point} updates, which is also an improvement of 1.0 BLEU points over the SFO Algorithm. 

\begin{table}[t]
  \caption{Test set evaluation for chunking under F1 score, and for machine translation under BLEU.  Results for stochastic learners are shown for SFO, and SZO algorithms with perturbation of all parameters (ALL) or for sparse perturbations of parameters for active features only (SPARSE). All results are averaged over three runs with different random seeds for each algorithm. Best SZO results are shown in bold face; higher results are better.}
\label{tab:test_results}
\centering
\begin{tabular}{lll}
  \toprule
  Update Rule & Test F1 Chunking & Test BLEU SMT \\
\midrule
Function Comparison (ALL) & 0.810 & 0.265 \\
Baseline Comparison (ALL) & 0.819 & 0.266 \\
Two-Point Evaluation (ALL) & 0.841 & \bf{0.273} \\
Function Comparison (SPARSE) & 0.842 & - \\
Baseline Comparison (SPARSE) & 0.869 & - \\
Two-Point Evaluation (SPARSE) & \bf{0.888} & - \\
SFO & 0.908 & 0.263 \\
\bottomrule
\end{tabular}
\end{table}

\paragraph{Further Experiments.} Further experiments using sparse, word-based models for multiclass text classification (which can be viewed as a degenarate structured prediction task) are given in the supplementary material.

\section{Conclusion}
\label{sec:conc}
The theoretical contribution of this paper is to show that the main bottleneck in SZO optimization---the dependency of the iteration complexity on the dimensionality of the function to be perturbed---can be reduced to the expected number of active features in sparse structured prediction scenarios. We presented experimental results on linear structured prediction tasks that confirm our theoretical results. Furthermore, we showed that it can be advantageous to use the same criterion at train and test time, something that is impossible for SFO algorithms if this criterion is non-differentiable.

The experiments in this paper were obtained by perturbing the discontinuous deterministic MAP criterion \eqref{eq:map}. Using a smooth annealed criterion \eqref{eq:annealed} yields similar results, and allows to match our experiments with our (and other existing) theory that assumes at least Lipschitz continuity of perturbed functions. To our knowledge, existing work on SZO optimization for discontinuous functions is concerned with a theory of differentiation and asymptotic convergence results \cite{MoreauAeyels:00,VicenteCustodio:12}, thus non-asymptotic convergence analysis for discontinuous SZO is an interesting open problem.

\subsubsection*{Acknowledgments.}

This research was supported in part by the German research foundation (DFG), and in part by a
research cooperation grant with the Amazon Development Center Germany.

\bibliographystyle{apalike}
\bibliography{references}

\begin{thebibliography}{}

\bibitem[Agarwal et~al., 2010]{AgarwalETAL:10}
Agarwal, A., Dekel, O., and Xiao, L. (2010).
\newblock Optimal algorithms for online convex optimization with multi-point
  bandit feedback.
\newblock In {\em {COLT}}, Haifa, Israel.

\bibitem[Bahdanau et~al., 2017]{BahdanauETAL:17}
Bahdanau, D., Brakel, P., Xu, K., Goyal, A., Lowe, R., Pineau, J., Courville,
  A., and Bengio, Y. (2017).
\newblock An actor-critic algorithm for sequence prediction.
\newblock In {\em {ICLR}}, Toulon, France.

\bibitem[Duchi et~al., 2015]{DuchiETAL:15}
Duchi, J.~C., Jordan, M.~I., Wainwright, M.~J., and Wibisono, A. (2015).
\newblock Optimal rates for zero-order convex optimization: The power of two
  function evaluations.
\newblock {\em {IEEE} Translactions on Information Theory}, 61(5):2788--2806.

\bibitem[Dyer et~al., 2010]{DyerETAL:10}
Dyer, C., Lopez, A., Ganitkevitch, J., Weese, J., Ture, F., Blunsom, P.,
  Setiawan, H., Eidelman, V., and Resnik, P. (2010).
\newblock cdec: A decoder, alignment, and learning framework for finite-state
  and context-free translation models.
\newblock In {\em {ACL Demo}}, Uppsala, Sweden.

\bibitem[Flaxman et~al., 2005]{FlaxmanETAL:05}
Flaxman, A.~D., Kalai, A.~T., and McMahan, H.~B. (2005).
\newblock Online convex optimization in the bandit setting: gradient descent
  without a gradient.
\newblock In {\em {SODA}}, Philadelphia, {PA}.

\bibitem[Ghadimi and Lan, 2012]{GhadimiLan:12}
Ghadimi, S. and Lan, G. (2012).
\newblock Stochastic first- and zeroth-order methods for nonconvex stochastic
  programming.
\newblock {\em {SIAM} Journal on Optimization}, 4(23):2342--2368.

\bibitem[Jamieson et~al., 2012]{JamiesonETAL:12}
Jamieson, K.~G., Nowak, R.~D., and Recht, B. (2012).
\newblock Query complexity of derivative-free optimization.
\newblock In {\em {NIPS}}, Lake Tahoe, {CA}.

\bibitem[Jang et~al., 2017]{JangETAL:17}
Jang, E., Gu, S., and Poole, B. (2017).
\newblock Categorical reparameterization with gumbel-softmax.
\newblock In {\em {ICLR}}, Toulon, France.

\bibitem[Keshet et~al., 2011]{KeshetETAL:11}
Keshet, J., McAllester, D., and Hazan, T. (2011).
\newblock {PAC-Bayesian} approach for minimization of phoneme error rate.
\newblock In {\em {ICASSP}}, Prague, Czech Republic.

\bibitem[Kiefer and Wolfowitz, 1952]{KieferWolfowitz:52}
Kiefer, J. and Wolfowitz, J. (1952).
\newblock Stochastic estimation of the maximum of a regression function.
\newblock {\em Annals of Mathematical Statistics}, 23(3):462--466.

\bibitem[Kingma and Welling, 2014]{KingmaWelling:14}
Kingma, D.~P. and Welling, M. (2014).
\newblock Auto-encoding variational bayes.
\newblock In {\em {ICLR}}, Banff, Canada.

\bibitem[Koehn and Schroeder, 2007]{KoehnSchroeder:07}
Koehn, P. and Schroeder, J. (2007).
\newblock Experiments in domain adaptation for statistical machine translation.
\newblock In {\em {WMT}}, Prague, Czech Republic.

\bibitem[Konda and Tsitsiklis, 2000]{KondaTsitsiklis:00}
Konda, V.~R. and Tsitsiklis, J.~N. (2000).
\newblock Actor-critic algorithms.
\newblock In {\em {NIPS}}, Vancouver, Canada.

\bibitem[Kreutzer et~al., 2017]{KreutzerETAL:17}
Kreutzer, J., Sokolov, A., and Riezler, S. (2017).
\newblock Bandit structured prediction for neural sequence-to-sequence
  learning.
\newblock In {\em {ACL}}, Vancouver, Canada.

\bibitem[Kushner and Yin, 2003]{KushnerYin:03}
Kushner, H.~J. and Yin, G.~G. (2003).
\newblock {\em Stochastic Approximation and Recursive Algorithms and
  Applications}.
\newblock Springer, second edition.

\bibitem[Mania et~al., 2018]{ManiaETAL:18}
Mania, H., Guy, A., and Recht, B. (2018).
\newblock Simple random search provides a competitive approach to reinforcement
  learning.
\newblock {\em CoRR}, abs/1803.07055.

\bibitem[McAllester and Keshet, 2011]{McAllesterKeshet:11}
McAllester, D. and Keshet, J. (2011).
\newblock Generalization bounds and consistency for latent structural probit
  and ramp loss.
\newblock In {\em {NIPS}}, Granada, Spain.

\bibitem[Moreau and Aeyels, 2000]{MoreauAeyels:00}
Moreau, L. and Aeyels, D. (2000).
\newblock Optimization of discontinuous functions: A generalized theory of
  differentiation.
\newblock {\em {SIAM} Journal on Optimization}, 11(1):53--69.

\bibitem[Nakov et~al., 2012]{NakovETAL:12}
Nakov, P., Guzm{\'a}n, F., and Vogel, S. (2012).
\newblock Optimizing for sentence-level bleu+1 yields short translations.
\newblock In {\em {COLING}}, Bombay, India.

\bibitem[Nesterov, 2004]{Nesterov:04}
Nesterov, Y. (2004).
\newblock {\em Introductory lectures on convex optimization: A basic course}.
\newblock Springer.

\bibitem[Nesterov and Spokoiny, 2015]{NesterovSpokoiny:15}
Nesterov, Y. and Spokoiny, V. (2015).
\newblock Random gradient-free minimization of convex functions.
\newblock {\em Foundations of Computational Mathematics}, (17):527--566.

\bibitem[Nguyen et~al., 2017]{NguyenETAL:17}
Nguyen, K., Daum\'{e}, H., and Boyd-Graber, J. (2017).
\newblock Reinforcement learning for bandit neural machine translation with
  simulated feedback.
\newblock In {\em {EMNLP}}, Copenhagen, Denmark.

\bibitem[Papineni et~al., 2002]{Papineni:02}
Papineni, K., Roukos, S., Ward, T., and Zhu, W.-J. (2002).
\newblock Bleu: a method for automatic evaluation of machine translation.
\newblock In {\em {ACL}}, Philadelphia, {PA}.

\bibitem[Plappert et~al., 2018]{PlappertETAL:18}
Plappert, M., Houthooft, R., Dhariwal, P., Sidor, S., Chen, R.~Y., Chen, X.,
  Asfour, T., Abbeel, P., and Andrychowicz, M. (2018).
\newblock Parameter space noise for exploration.
\newblock In {\em {ICLR}}, Vancouver, Canada.

\bibitem[Ross, 2013]{Ross:13}
Ross, S.~M. (2013).
\newblock {\em Simulation}.
\newblock Elsevier, fifth edition.

\bibitem[Salimans et~al., 2017]{SalimansETAL:17}
Salimans, T., Ho, J., Chen, X., and Sutskever, I. (2017).
\newblock Evolution strategies as a scalable alternative to reinforcement
  learning.
\newblock {\em CoRR}, abs/1703.03864.

\bibitem[Sehnke et~al., 2010]{SehnkeETAL:10}
Sehnke, F., Osendorfer, C., R{\"u}ckstie{\ss}, T., Graves, A., Peters, J., and
  Schmidhuber, J. (2010).
\newblock Parameter-exploring policy gradients.
\newblock {\em Neural Networks}, 23(4):551--559.

\bibitem[Sha and Pereira, 2003]{ShaPereira:03}
Sha, F. and Pereira, F. (2003).
\newblock Shallow parsing with conditional random fields.
\newblock In {\em {HLT-NAACL}}, Edmonton, Cananda.

\bibitem[Shamir, 2013]{Shamir:13}
Shamir, O. (2013).
\newblock On the complexity of bandit and derivative-free stochastic convex
  optimization.
\newblock In {\em {COLT}}, Princeton, {NJ}.

\bibitem[Smith and Eisner, 2006]{SmithEisner:06}
Smith, D.~A. and Eisner, J. (2006).
\newblock Minimum risk annealing for training log-linear models.
\newblock In {\em {COLING-ACL}}, Sydney, Australia.

\bibitem[Sokolov et~al., 2016]{SokolovETALnips:16}
Sokolov, A., Kreutzer, J., Lo, C., and Riezler, S. (2016).
\newblock Stochastic structured prediction under bandit feedback.
\newblock In {\em {NIPS}}, Barcelona, Spain.

\bibitem[Spall, 1992]{Spall:92}
Spall, J.~C. (1992).
\newblock Multivariate stochastic approximation using a simultaneous
  perturbation gradient approximation.
\newblock {\em {IEEE} Transactions on Automatic Control}, 37(3):332--341.

\bibitem[Spall, 2003]{Spall:03}
Spall, J.~C. (2003).
\newblock {\em Introduction to Stochastic Search and Optimization: Estimation,
  Simulation, and Control}.
\newblock Wiley.

\bibitem[Sutton and Barto, 2017]{SuttonBarto:17}
Sutton, R.~S. and Barto, A.~G. (2017).
\newblock {\em Reinforcement Learning. An Introduction}.
\newblock The {MIT} Press, second edition.

\bibitem[Sutton et~al., 2000]{SuttonETAL:00}
Sutton, R.~S., McAllester, D., Singh, S., and Mansour, Y. (2000).
\newblock Policy gradient methods for reinforcement learning with function
  approximation.
\newblock In {\em {NIPS}}, Vancouver, Canada.

\bibitem[Vicente and Custodio, 2012]{VicenteCustodio:12}
Vicente, L. and Custodio, A. (2012).
\newblock Analysis of direct searches for discontinous functions.
\newblock {\em Mathematical Programming}, 133(1-2):229--325.

\bibitem[Williams, 1992]{Williams:92}
Williams, R.~J. (1992).
\newblock Simple statistical gradient-following algorithms for connectionist
  reinforcement learning.
\newblock {\em Machine Learning}, 20:229--256.

\bibitem[Yue and Joachims, 2009]{YueJoachims:09}
Yue, Y. and Joachims, T. (2009).
\newblock Interactively optimizing information retrieval systems as a dueling
  bandits problem.
\newblock In {\em {ICML}}, Montreal, Canada.

\end{thebibliography}

\section*{Supplementary Material}
\appendix
\section{Theorems and Proofs}

\newtheorem{innercustomlemma}{Lemma}
\newenvironment{customlemma}[1]
  {\renewcommand\theinnercustomlemma{#1}\innercustomlemma}
  {\endinnercustomlemma}

\newtheorem{innercustomcorr}{Corollary}
\newenvironment{customcorr}[1]
  {\renewcommand\theinnercustomcorr{#1}\innercustomcorr}
  {\endinnercustomcorr}

\begin{customlemma}{1}
  \begin{align}
    \nabla_w f_\mu(w) & = \E_{\bu,x}[\frac{F(w + \mu \bu,x) - F(w,x)}{\mu}\bu]. \label{eq:2-point}
    \end{align}
  \end{customlemma}

\begin{proof}
\begin{align*}
\nabla_w f_\mu(w) &= \nabla_w \E_x \left[\int_{\bu} \N(\bu)\ F(w+\mu \bu,x)\ d\bu\ \right]\\
&= \E_x \left[ \int_{\bu} \ \nabla_w \N(\bu)\ F(w + \mu \bu,x)\ d\bu\ \right]\\
&= \E_x \left[\int_{\bu} \nabla_w\frac{1}{\sqrt{(2\pi)^n}}e^{-\frac{1}{2}||\bu||^2} F(w + \mu \bu,x) d\bu\ \right]\\
&= \E_x \left[\int_y \nabla_w\frac{1}{\sqrt{(2\pi)^n}}e^{-\frac{1}{2}||\frac{y - w}{\mu}||^2} F(y,x) \cdot \frac{1}{\mu^n}dy\ \right] \;\;\; \text{where $y=w+\mu \bu$} \\
&= \E_x \left[\int_y 1\cdot\frac{y-w}{\mu^2}\frac{1}{\sqrt{(2\pi)^n}}e^{-\frac{1}{2}||\frac{y - w}{\mu}||^2} F(y,x) \cdot\frac{1}{\mu^n}dy\ \right]\\
&= \E_x \left[\int_{\bu} \frac{\bu}{\mu}\frac{1}{\sqrt{(2\pi)^n}}e^{-\frac{1}{2}||\bu||^2} F(w+\mu \bu,x) d\bu\ \right]\\
&= \E_x \left[\int_{\bu} \frac{\bu}{\mu}\ \N(\bu)\ F(w+\mu \bu,x) d\bu\ \right]\\
&= \E_x \left[\int_{\bu} \frac{\bu}{\mu}\ \N(\bu)\ F(w+\mu \bu,x) - F(w,x) d\bu\ \right].
\end{align*}
The last line follows since $\E_{x,\bu}[\frac{F(w,x)}{\mu}\bu] = \frac{F(w,x)}{\mu}\E_{\bu,x}[\bu]=0$.
\end{proof}

\begin{customcorr}{1}
  \begin{align}
    \E_{\bar{\mathcal{U}}_N,\mathcal{X}_N }[\norm{\nabla f_\mu(w_{N})}^2] \leq \epsilon \text{ if } N \geq \bigO{\big(\frac{\bar{n}^3}{\epsilon^2}\big)}.
      \end{align}
\end{customcorr}

\begin{proof}
  Our goal is to bound the terms on the righthandside of Theorem 1 in order to classify Algorithm 1 by the number of iterations necessary to guarantee an accuracy of $\epsilon$ for the expected squared norm of the gradient of $f_\mu$. The inequality is repeated here:
  \begin{align}
    \label{thm1}
  \frac{1}{S_N} \sum_{k=0}^N h_k \E_{\bar{\mathcal{U}}_k,\mathcal{X}_k }\big[\norm{\nabla{f_\mu(w_k)}^2}\big] \leq \frac{1}{S_N} \left( (f_\mu(w_0) - f^\ast) + \left(\frac{1}{2} L_1 (\bar{n}+4)^2 L_0^2 \right) \sum_{k=0}^N h_k^2 \right).
  \end{align}
  
  We follow \cite{NesterovSpokoiny:15} in assuming a constant learning rate $h_k := h, \; k \geq 0$. 

  We bound the approximation gap between the original function $f$ and the smoothed approximation $f_\mu$ by $\alpha$ choosing $\mu \leq \tilde{\mu} = \frac{\alpha}{\bar{n}^{1/2}L_0}$. The latter is possible by applying \cite{NesterovSpokoiny:15}, eq. (18), to our case s.t. $|f_\mu(w) - f(w)| \leq \mu L_0\bar{n}^{1/2}$.

  Furthermore, we apply \cite{NesterovSpokoiny:15}, eq. (22) to our case s.t. $L_1 = \frac{\bar{n}^{1/2}}{\tilde{\mu}}L_0$ where $L_0(F(\cdot,x)) \leq L_0$ for all $F(\cdot,x), x \in \mathcal{X}$.
  
  Let $S_N = \sum_{k=0}^N h_k = (N+1)h$, then the right-hand side of eq. \eqref{thm1} becomes
  \begin{align}
    & \frac{1}{(N+1)h}  \left((f_{\tilde{\mu}}(w_0) - f^\ast)  + \frac{1}{\tilde{\mu}}
    \bar{n}^{1/2} (\bar{n}+4)^2 L_0^3 (N+1)h^2 \right) \\
    & = \frac{1}{(N+1)h} \left((f_{\tilde{\mu}}(w_0) - f^\ast) + \frac{\bar{n}^{1/2} L_0}{\alpha} \bar{n}^{1/2} (\bar{n}+4)^2 L_0^3 (N+1)h^2 \right) \\
    & = \frac{f_{\tilde{\mu}}(w_0) - f^\ast}{(N+1)h} + \frac{h}{\alpha}\bar{n} (\bar{n}+4)^2 L_0^4 \\
    & \leq \frac{L_0 R}{(N+1)h} + \frac{h}{\alpha}\bar{n} (\bar{n}+4)^2 L_0^4. 
  \end{align}
  The last inequality follows from Lipschitz continuity s.t. $f_{\tilde{\mu}}(w_0) - f^\ast \leq L_0 \norm{w_0 - w^\ast}$ and the additional assumption of a bound $\norm{w_0 - w^\ast} \leq R$. Minimizing this upper bound in $h$, by taking the first derivative and setting it to zero, gives
  \begin{align}
    h^\ast = \left( \frac{\alpha R}{\bar{n}(\bar{n}+4)^2 L_0^3 (N+1)} \right)^{1/2}.
    \end{align}
  Plugging this back into the upper bound gives
  \begin{align}
    & \frac{L_0 R}{(N+1)h^\ast} + \frac{h^\ast}{\alpha}\bar{n} (\bar{n}+4)^2 L_0^4 \\
    & = 2 \left(\frac{\bar{n}(\bar{n}+4)^2 L_0^5 R}{\alpha (N+1)} \right)^{1/2}.
    \end{align}
    Thus, in order to guarantee an $\epsilon$-accuracy on the left-hand side of eq. \eqref{thm1}, we need $N \geq \bigO(\frac{\bar{n}^3}{\epsilon^2})$ iterations.
  \end{proof}

\begin{figure}[t]
    \centering
  \includegraphics[width=0.5\columnwidth]{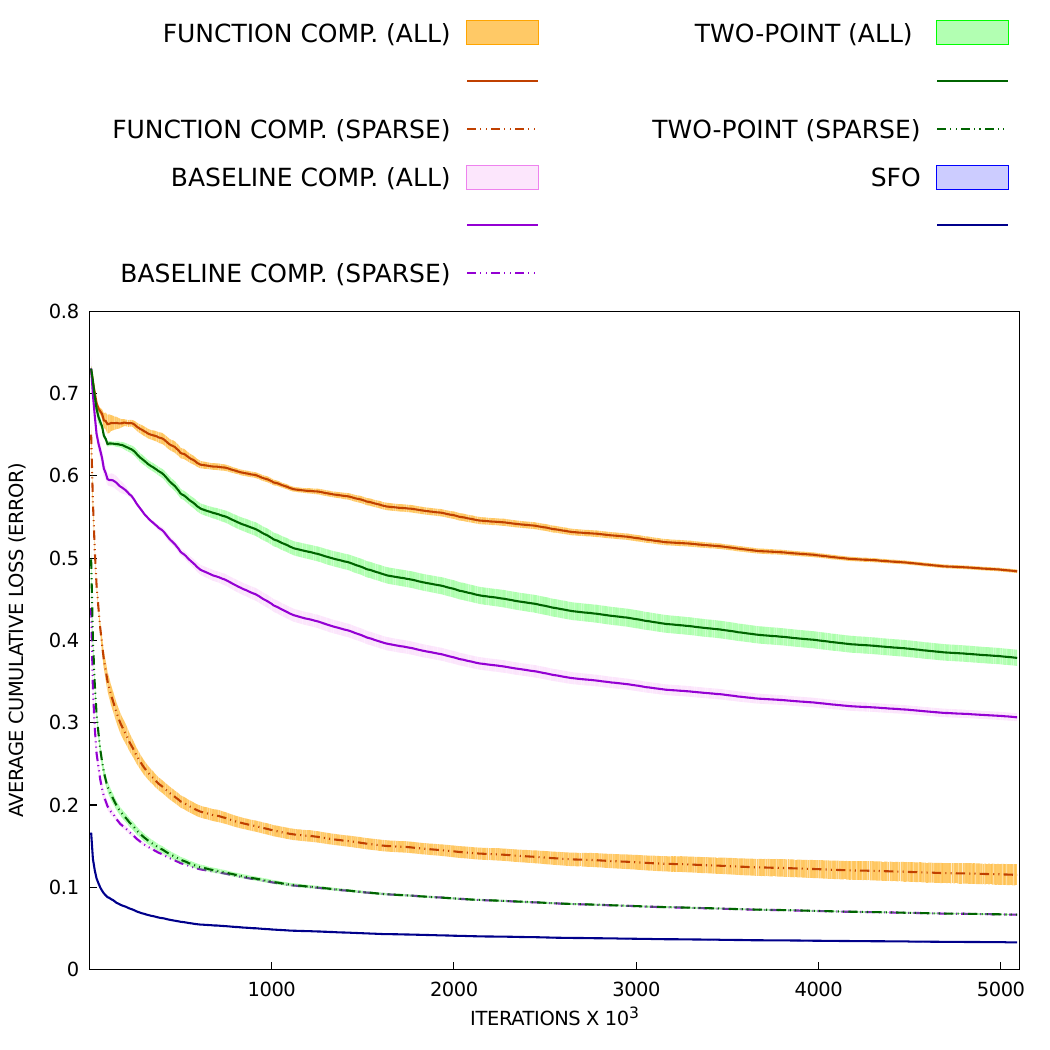}
    \caption{Average cumulative loss on training data for sparse multiclass text classification. All results are averaged over 3 runs, with mean results in bold lines, 2 standard deviations in filled areas.}
    \label{fig:avg_loss_multiclass}
\end{figure}

\section{Further Experiments}

\paragraph{Sparse Models for Multiclass Text Classification.} Multiclass text classification uses a sparse linear feature representation on the Reuters RCV1 dataset \cite{LewisETAL:04}. This dataset a standard benchmark for (simplified) structured prediction that has been used in a bandit setup by \cite{KakadeETAL:08}. The simplified problem uses a binary $\Delta$ function indicating incorrect assignment of one out of 4 classes.  The data were split into a training set (509,381 documents from original \texttt{test\_pt[0-2].dat} files), a development set (19,486 docs: every 8th entry from \texttt{test\_pt3.dat} and a test set (19,806 docs from \texttt{train.dat}). Training for bandit learning is done by cold starting the models from $w_0=\mathbf{0}$. Meta-parameter settings were determined on development sets for constant learning rate $h$ in the range of $10^0$ to $10^{-5}$, and exploration parameter $\mu$ in the range of $10^{-2}$ to $10^{-6}$.
   
Following~\cite{KakadeETAL:08}, we used documents with exactly one label from the set of labels \textsc{\{ccat, ecat, gcat, mcat\}} and converted them to \emph{tfidf} word vectors of dimension 227,903 on the training set.
This \emph{tfidf} conversion yields very sparse features with a sparsity pattern of on average 0.5\%.

As shown in Figure \ref{fig:avg_loss_multiclass}, best convergence behavior for sparse multiclass text classification is obtained by the SFO method which functions as upper bound for the SZO methods. Among SZO methods we see a clear grouping of algorithms with SPARSE perturbation of active features only (dashed curves) and standard SZO methods where ALL parameters are perturbed (dense curves), with a clear advantage in convergence speed for the former. A comparison of the different SZO update rules defined in Section \ref{sec:algo} shows a similar ranking to the experiments described above, with updates based on (SPARSE) two-point function evaluation converging fastest.

\end{document}